\documentclass[12pt]{amsart} 

\usepackage{amsmath, amssymb, amsthm}  
\usepackage{hyperref}  
\usepackage{graphicx}  
\usepackage{cite}      
\usepackage{geometry}  
\usepackage[numbers]{natbib}
\usepackage{enumerate} 
\usepackage[all]{xy} 
\theoremstyle{plain}  
\newtheorem{theorem}{Theorem}[section]
\newtheorem{lemma}[theorem]{Lemma}
\newtheorem{proposition}[theorem]{Proposition}

\theoremstyle{definition}  
\newtheorem{definition}[theorem]{Definition}
\newtheorem{assumption}[theorem]{Assumption}  
\newtheorem{remark}{Remark}
\DeclareMathOperator*{\esssup}{ess\,sup}

\title[Learning with $\beta$- and $\phi$-mixing sequence]{Gradient Descent Algorithm in Hilbert Spaces under Stationary Markov Chains\\with $\phi$- and $\beta$-Mixing
}
\author{Priyanka Roy}
\address{Institute for Mathematical Methods in Medicine and Data-Based Modeling\\
         Johannes Kepler University Linz, Altenberger Strasse 69, A-4020 Linz, Austria}
\email{priyanka.roy@jku.at}
\author{Susanne Saminger-Platz}
\address{Institute for Mathematical Methods in Medicine and Data-Based Modeling\\
         Johannes Kepler University Linz, Altenberger Strasse 69, A-4020 Linz, Austria}
\email{susanne.saminger-platz@jku.at}


\date{\today}

\def\tsc#1{\csdef{#1}{\textsc{\lowercase{#1}}\xspace}}
\tsc{WGM}
\tsc{QE}
\tsc{EP}
\tsc{PMS}
\tsc{BEC}
\tsc{DE}

\begin{document}

\begin{abstract}
In this paper, we study a strictly stationary Markov chain gradient descent algorithm operating in general Hilbert spaces. Our analysis focuses on the mixing coefficients of the underlying process, specifically the $\phi$- and $\beta$-mixing coefficients. Under these assumptions, we derive probabilistic upper bounds on the convergence behavior of the algorithm based on the exponential as well as the polynomial decay of the mixing coefficients.
\end{abstract}
\maketitle

\section{Introduction}
Let \( W \) be a Hilbert space, and let \( (z_t)_{t \in \mathbb{N}} \) be a strictly stationary Markov chain on the measurable space \( (Z, \mathcal{B}(Z)) \), with transition kernel \( P \) and unique stationary distribution \( \rho \). Furthermore, we consider that the decay of dependence exhibited in the Markov chain is characterized by certain mixing coefficients, specifically \(\phi\)- and \(\beta\)-mixing. Given a quadratic loss function \( V : W \times Z \to \mathbb{R} \), we denote its gradient with respect to the first argument by \( \nabla V_z(w) \). Starting from an initial point \( w_1 \in W \), we define an iterative sequence by

\[
    w_{t+1} = w_t - \gamma_t \nabla V_{z_t}(w_t), \quad t \in \mathbb{N},
\]
where \( (\gamma_t)_{t \in \mathbb{N}} \) is a positive step-size sequence. Based on the Markov samples \( (z_t)_{t \in \mathbb{N}} \), our primary goal is to investigate the convergence behavior of the sequence \( (w_t)_{t \in \mathbb{N}} \) towards the unique minimizer \( w^\star \in W \) of the quadratic loss function \( V \), satisfying
\(
    \nabla V(w^\star) = 0,
\) which takes into account the mixing coefficients of the Markov chain, specifically the \(\phi\)- and \(\beta\)-mixing coefficients.
\par
For both cases of \(\phi\)- and \(\beta\)-mixing coefficients, a summary of our findings is that for parameter values 
\( \theta \in \left( \tfrac{1}{2}, 1 \right) \), the convergence rate remains 
\( \mathcal{O}\!\left( t^{-\theta/2} \right) \), aligning with the i.i.d.\ rates established by Smale and Yao \cite{MR2228737}. 
However, at the boundary case \( \theta = 1 \), the rate deteriorates to 
\( \mathcal{O}\!\left( t^{-\alpha/2} \right) \), in contrast to the i.i.d.\ setting where a faster rate 
\( \mathcal{O}\!\left( t^{-\alpha} \right) \) is achieved, as shown by Smale and Yao \cite{MR2228737}, where \(\alpha \in (0,1]\).
\par
As for the polynomial decay of both the mixing coefficients, for example, when \(\phi_i\leq bi^{-k}\), where \(b>0\) and \(k>0\), for parameter values 
\( \theta \in \left( \tfrac{1}{2}, 1 \right) \), the convergence rate remains the same as that of the i.i.d. rate for the value \(k>1\), however for \(k=1\), the rate remains almost the same as that of the i.i.d. rate except for a logarithmic factor, i.e., \(O\left(t^{-\theta/2} (\log t)^{1/2} \right)\).
\par
Moreover, our results are sharp in the sense that when the mixing rate is fast, the dependence becomes negligible, making the error bounds and learning rates nearly identical to those observed in the independent sample setting.
\par
This approach defines a variant of the stochastic gradient descent algorithm, which we refer to as strictly stationary Markov chain gradient descent (SS-MGD). Unlike conventional Markov chain gradient descent methods that emphasize mixing time and convergence toward the stationary distribution see, e.g., \cite{sun2018Markov,nagaraj2020least,dorfman2022adapting,even2023stochastic}, SS-MGD assumes that the Markov chain is strictly stationary from initialization, i.e., the initial distribution is already the invariant distribution \(\rho\).
The version with the independence assumption of this algorithm in Hilbert spaces was studied by Smale and Yao~\cite{MR2228737}. 
In many scenarios such as time series, Markov chains, stochastic processes e.t.c., the i.i.d.\ assumption might be violated due to the exhibited temporal dependence, see e.g., \cite{mokkadem1988mixing, tuan1985some, beeram2021survey}. Despite this, learning algorithms have been successfully applied in such dependent settings, thereby motivating the need for a more theoretical foundation to understand their performance under these conditions. Various measures of statistical dependence are relevant in this context, for e.g., mixing coefficients (e.g., \(\alpha\)-, \(\beta\)-, and \(\phi\)-mixing), as well as spectral properties such as the spectral gap and mixing time of Markov chains. 
\par
There has been significant interest in analyzing learning algorithms under various dependent sequences, especially focusing on how various mixing coefficients influence performance. Early foundational work by Meir~\cite{meir2000nonparametric} extended the Vapnik–Chervonenkis framework to dependent, $\beta$-mixing time series, providing finite-sample risk bounds and pioneering the rigorous use of structural risk minimization (SRM) in dependent settings. Subsequently, Modha and Masry~\cite{modha2002minimum} generalized minimum complexity regression estimators from i.i.d.\ data to dependent scenarios. Notably, their convergence rates matched i.i.d.\ cases for \(m\)-dependent observations but were slower under strong mixing due to dependence-induced limitations on effective sample size. Later contributions further clarified how dependence affects learning rates. For instance, Zou and Li~\cite{zou2007performance} derived convergence bounds for empirical risk minimization with exponentially strongly mixing sequences. Mohri and Rostamizadeh~\cite{mohri2008rademacher, mohri2010stability} broadened the scope by introducing Rademacher complexity-based and algorithm-dependent stability bounds for stationary $\beta$- and $\phi$-mixing sequences, employing independent-block techniques. Similarly, Ralaivola \emph{et al.}~\cite{ralaivola2010chromatic} proposed PAC-Bayes bounds tailored for $\beta$- and $\phi$-mixing data, using dependency graph partitioning to handle data interdependencies explicitly. Regularized regression algorithms have also seen several advancements under various mixing conditions. For instance, Xu and Chen \cite{MR2406432} showed that when samples satisfy exponential \(\alpha\)-mixing, the performance of Tikhonov regularization depends on an effective sample size. Sun and Wu \cite{MR2581234} derived capacity-independent bounds for least-squares regression under both \(\alpha\)- and \(\phi\)-mixing, proving that apart from a logarithmic factor, convergence rates match those in the i.i.d. case. Furthermore, Steinwart \emph{et al.} \cite{steinwart2009learning}, Steinwart and Christmann \cite{steinwart2009fast}, and Xu \emph{et al.} \cite{xu2014generalization} analyzed SVM consistency and convergence rates under various mixing processes, highlighting how mixing rates affect optimality. More recently, Tong and Ng~\cite{tong2024spectral} recovered near-optimal i.i.d.\ convergence rates for strongly mixing sequences in spectral learning scenarios. In contrast, Zou \emph{et al.} \cite{MR2886182} and Duchi \emph{et al.} \cite{agarwal2012generalization} focused on the generalization capabilities of empirical risk minimization and online learning algorithms under stationary mixing processes, deriving explicit performance guarantees.
\par
In contrast to previous analyses that predominantly focused on empirical risk minimization, regularized regression, etc., under various mixing conditions, our work specifically investigates the performance of stochastic gradient descent algorithms under strictly stationary Markov chain assumptions, where the dependence in the chain is characterized by either $\phi$- or $\beta$-mixing coefficients. We explicitly analyze the learning capabilities of this algorithm by deriving precise convergence rates, thereby extending the understanding of how certain dependencies exhibited in the Markov chain impact algorithmic performance in stochastic optimization settings. The concept of such stochastic approximations was first introduced by Robbins and Monro \cite{MR42668} and Kiefer and Wolfowitz \cite{MR50243}. Their convergence rates and asymptotic properties have been investigated in several works (see, e.g., \cite{MR1767993}, \cite{MR595479}, and \cite{MR2044089}). For a comprehensive background on stochastic algorithms, we refer the reader to Duflo \cite{MR1612815} and Kushner and Yin \cite{MR1993642}.
\par
To outline our paper, we begin in Section~\ref{prelim} with a brief overview of the $\beta$- and $\phi$-mixing coefficients, which characterize the dependence structure in stationary processes, followed by Section~\ref{section3} which introduces a strictly stationary Markov chain gradient descent (SS-MGD) algorithm in Hilbert spaces and extends its analysis to the setting of exponentially $\phi$-mixing Markov chains, resulting in Theorem~\ref{thm:1}. Section~\ref{betamix} discusses analogous results for $\beta$-mixing sequences. Finally, Section~\ref{poly} presents convergence rates of SS-MGD under polynomially decaying mixing coefficients, and Appendix~\ref{appen} compiles several established inequalities that are essential for our analysis.
\section{The measures of dependence (Mixing coefficients)}\label{prelim}
In this section, we recall some relevant definitions, properties, and established results concerning some of the mixing coefficients (see for e.g.,  \cite{MR2178042, MR2325294, MR2325295, MR2325296}).
\par
Denote $(\Omega, \mathcal{F}, P)$ as a probability space. For any $\sigma$-field $\mathcal{A} \subseteq \mathcal{F}$, let $L^2_{\text{real}}(\mathcal{A})$ denote the space of (equivalence classes of) square-integrable $\mathcal{A}$-measurable (real-valued) random variables. For any two $\sigma$-fields $\mathcal{A}$ and $\mathcal{B} \subseteq \mathcal{F}$, we will focus on the following two measures of dependence
\begin{align*}  
    \phi(\mathcal{A}, \mathcal{B}) &:= \sup \left| P(B \mid A) - P(B) \right|, \quad \forall  A \in \mathcal{A} \text{ with }P(A) > 0, \; \forall B \in \mathcal{B}, \;  \\
    \beta(\mathcal{A}, \mathcal{B}) &:= \sup \frac{1}{2} \sum_{i=1}^{I} \sum_{j=1}^{J} \left| P(A_i \cap B_j) - P(A_i)P(B_j) \right|, \quad \forall A \in \mathcal{A}, \; \forall B \in \mathcal{B},   
\end{align*}
where the supremum is taken over all pairs of (finite) partitions \(\{A_1, \ldots, A_I\}\) and\\\(\{B_1, \ldots, B_J\}\) of \(\Omega\) such that \(A_i \in \mathcal{A}\) for each \(i \in I\) and \(B_j \in \mathcal{B}\) for each \(j\in J\).\\
For independent \(\mathcal{A}\) and \(\mathcal{B}\) we obtain \begin{align*}
     \phi(\mathcal{A}, \mathcal{B}) = 0,  \quad \beta(\mathcal{A}, \mathcal{B}) = 0.
\end{align*}
Moreover, these measures of dependence satisfy the following inequalities
\begin{align}
    0\leq \beta(\mathcal{A}, \mathcal{B}) \leq \phi(\mathcal{A}, \mathcal{B})\leq 1\label{r1}
\end{align}
We now recall the definition of some mixing coefficients of a strictly stationary Markov chain.
\begin{definition}
    Let $(Z_t)_{t\in\mathbb{N}}$ be a strictly stationary Markov chain; then\\
$\beta_t = \beta(\sigma(Z_0), \sigma(Z_t))$ and $\phi_t = \phi(\sigma(Z_0), \sigma(Z_t))$.
Looking at the random sequence $(Z_t)_{t\in\mathbb{N}}$ for $t\to\infty$, we say that it is \textbf{$\phi$-mixing} if $\phi_t \to 0$ and \textbf{absolutely regular (or $\beta$-mixing)} if $\beta_t \to 0$.

Using the conditional probabilities \( P^t(z, B) = P(z_t \in B \mid z_0 = z) \) and the invariant distribution \(\rho \) (i.e., we take the starting distribution \(\mu_0=\rho\) since we consider a stationary Markov chain), the mixing coefficients can equivalently be expressed as follows (also see \cite[Theorem 3.32]{MR2325294}, \cite{MR1312160}, \cite{MR2944418})

\begin{equation}
  \beta_t = \int_{Z} \sup_{B \in \mathcal{B}(Z)} \left| P^t(z, B) - \rho(B) \right|\,d\rho,  
\end{equation}

and
\begin{equation}
   \phi_t = \sup_{B \in \mathcal{B}(Z)} \, \text{ess} \sup_{z \in Z} \left| P^t(z, B) - \rho(B) \right|. 
\end{equation}
\end{definition}

\subsection{Example}
Stochastic processes that satisfy mixing conditions include, for instance, ARMA processes, copula-based Markov chains, etc. Mokkadem \cite{mokkadem1988mixing} demonstrates that, under mild assumptions specifically, the absolute continuity of innovations, every stationary vector ARMA process is geometrically completely regular. Similarly, Longla and Peligrad \cite{MR2944418} show that for copula-based Markov chains, certain properties such as a positive lower bound on the copula density ensure exponential \(\phi\)-mixing, which in turn implies geometric ergodicity. Such copula-based Markov chains have found applications in time series econometrics and other applied fields. For additional examples of mixing processes, see the foundational work of Davydov \cite{davydov1973mixing}.

\section{Strictly stationary Markov chain based gradient descent algorithm in Hilbert spaces}\label{section3}
 We consider a strictly stationary Markov chain gradient descent algorithm in general Hilbert spaces, which is an extension of a stochastic gradient algorithm in Hilbert spaces as discussed by Smale and Yao \cite{MR2228737}.
Let \( W \) be a Hilbert space with inner product \( \langle \cdot, \cdot \rangle \). Consider a quadratic potential \( V: W \rightarrow \mathbb{R} \) given by
\[
V(w) = \frac{1}{2} \langle Aw, w \rangle + \langle B, w \rangle + C,
\]
where \( A: W \rightarrow W \) is a positive definite bounded linear operator with bounded inverse, i.e., \( \|A^{-1}\| < \infty \), \( B \in W \), and \( C \in \mathbb{R} \). Let \( w^{\star} \in W \) denote the unique minimizer of \( V \) such that \( \nabla V(w^{\star}) = 0 \), where \(\nabla V\) is the gradient of \(V\) i.e., \(\nabla(V):W \rightarrow W\) given as \(\nabla V(w)=Aw+b\). Moreover, for each sample \(z\), \(\nabla V_z:W\rightarrow W\) is given by the affine map \(\nabla V_z(w)=A(z)w+B(z)\), with \(A(z)\), \(B(z)\) denoting the values of the random variables \(A\), \(B\) at \(z\in Z\). 

We aim to approximate \( w^{\star} \) by adapting the Markov chain-based stochastic gradient algorithm to strictly stationary chain taking into account the associated \(\beta\)- and \(\phi\)-mixing coefficients and analyze its sample complexity. Let \( (z_t)_{t\in\mathbb{N}} \) be a strictly stationary Markov chain on an uncountable state space \( (Z, \mathcal{B}(Z)) \) with transition kernel \( P \) and unique stationary distribution \( \rho \). We define an update formula for \( t = 1, 2, 3, \ldots \) by
\begin{align}
    w_{t+1} &= w_t - \gamma_t \nabla V_{z_t}(w_t), \quad \text{for some } w_1 \in W, \label{eq:stograd} 
\end{align}
where \( \nabla V_{z_t}: W \rightarrow W \) depends on sample \( z_t \), and \( (\gamma_t)_{t\in\mathbb{N}} \) is a positive sequence of step sizes. With slight abuse of notation, we shall also write $\nabla V_z$ whenever the reference to a sample $z\in Z$ rather than to a time step $t$ in a sampling or update process seems to be more appropriate. We additionally assume that \(\nabla V_z\) is Bochner integrable i.e., \(\int_Z \|\nabla V_z\|_W\,d\rho<\infty\), where \(\rho\) is a  measure on $Z$.
\par
Following is an example of a quadratic potential loss function. Let \(W=\mathcal{H}_K\), where \(\mathcal{H}_K\) is a reproducing kernel Hilbert spaces. Hence, for a fixed \(z=(x,y)\in Z\) we consider the quadratic potential loss function \(V:\mathcal{H}_K\rightarrow \mathbb{R}\) given as 
\[V_z(f)=\frac{1}{2}\{(f(x)-y)^2+\lambda\|f\|^2_K\}.\]
Then, it can be shown that for any $f\in \mathcal{H}_K,$ we have 
\(A_x(f)=f(x)K_x+\lambda f\) such that after taking the expectation of \( A_x \) over \( x \), we observe that
\[
\hat{A} = \mathbb{E}_x \left[ A_x \right] = L_K + \lambda I,
\]
and 
\(B_z=-yK_x\) and taking the expectation of \( B_z \) over \(z\), we observe that 
\[
\hat{B}= -L_K f_\rho,
\]
where \(\lambda>0\) is known as the regularization parameter, \(L_K:\mathcal{H}_K\rightarrow\mathcal{H}_K\) is a linear map, and for probability measure \(\rho\) on \( Z = X \times Y \), the regression function \(f_{\rho}\) is given as 
\[
f_{\rho}(x) = \int_Y y \, d\rho(y \mid x),
\]
where \(\rho(y \mid x)\) denotes the conditional distribution of \(y\) given \(x\). It has been be shown that the regression function \( f_{\rho} \) minimizes the expected mean squared error defined by
\[
\mathcal{E}(f) = \int_{X \times Y} (f(x) - y)^2 \, d\rho(z).
\]
For further information, see Smale and Yao \cite{MR2228737}.
\par
Next, in order to analyze the convergence, we introduce the following assumptions

\begin{assumption}\label{A1}
The function \( V \) has a unique minimizer \( w^\star \), and for all \( z \in Z \), there exists a constant \( \sigma \geq 0 \) such that
\[
\|\nabla V_z(w^\star)\|^2 \leq \sigma^2.
\]
\end{assumption}

\begin{assumption}\label{A2}
For all \( z \in Z \), the function \( V_z \) are \( \eta \)-smooth and \( \kappa \)-strongly convex, i.e., for all \( w \in W \)
\[
\kappa  \leq \nabla^2 V_z(w) \leq \eta .
\]
Define $\alpha=\frac{\kappa}{\eta}$, then $\alpha \in (0,1].$
\end{assumption}
Note that Assumption \ref{A1} reflects the noise at optimum with mean 0 i.e., $\mathbb{E}[\nabla V_z(w^\star)]=0.$ We now provide an upper bound on the distance between sample-based model $w_t$ based on a strictly stationary Markov chain which is \(\phi\)-mixing  at least exponentially fast, and the target model $w^\star$ in Hilbert spaces which results in Theorem \ref{thm:1} and Proposition \ref{rm:1}.
\par
Before proceeding to the statement and proof of Theorem~\ref{thm:1}, we first establish the necessary preliminary results, namely Propositions~\ref{prop1} and~\ref{prop2}, which lay the foundation for the subsequent argument which provide the groundwork for the proof of the Theorem~\ref{thm:1}.
\par
Note that throughout, we adopt a slight abuse of notation; i.e., we use $\nabla V_z$ to emphasize dependence on a sample $z \in Z$, and $\nabla V_t$ to emphasize dependence on the time step $t$, referring to $z_t$ when contextually appropriate.
\begin{proposition}
\label{prop1}
Let $(z_t)_{t\in\mathbb{N}}$ on $(Z, \mathcal{B}(Z))$ be a strictly stationary Markov chain. Let $\theta \in \left(\frac{1}{2},1\right)$. Then under Assumptions \ref{A1} and \ref{A2} we have that
\[\mathbb{E} \left[\sum_{i=1}^t \frac{1}{i^{2\theta}} \prod_{k=i+1}^t \left( 1 - \frac{\alpha}{k^\theta} \right)^2 \| \nabla V_{i}(\omega^*) \|^2\right] \leq \sigma^2 C_{\theta}\left( \frac{1}{\alpha} \right)^{\theta / (1-\theta)} \left( \frac{1}{t+1} \right)^\theta,\]
where $ C_{\theta}= \left(8 + \frac{2}{2\theta - 1} \left( \frac{\theta}{e(2 - 2^\theta)} \right)^{\theta / (1-\theta)}\right).$
\end{proposition}
\begin{proof}
Let us denote \(S_1=\mathbb{E} \left[\sum_{i=1}^t \frac{1}{i^{2\theta}} \prod_{k=i+1}^t \left( 1 - \frac{\alpha}{k^\theta} \right)^2 \| \nabla V_{i}(\omega^*) \|^2\right]\). Therefore we observe that 
\begin{align*}
S_1&= \sum_{i=1}^t \frac{1}{i^{2\theta}} \prod_{k=i+1}^t \left( 1 - \frac{\alpha}{k^\theta} \right)^2\mathbb{E}\left[ \| \nabla V_{i}(\omega^*) \|^2\right]\\
&= \sum_{i=1}^t \frac{1}{i^{2\theta}} \prod_{k=i+1}^t \left( 1 - \frac{\alpha}{k^\theta} \right)^2\biggl(\int_Z \|\nabla V_{z}(w^\star)\|^2d\rho(z) \biggl) \\
&\leq \sigma^2 \sum_{i=1}^t \frac{1}{i^{2\theta}} \prod_{k=i+1}^t \left( 1 - \frac{\alpha}{k^\theta} \right)^2   \\
&\leq \sigma^2 C_{\theta}\left( \frac{1}{\alpha} \right)^{\theta / (1-\theta)} \left( \frac{1}{t+1} \right)^\theta,
\end{align*}
where $ C_{\theta}= \left(8 + \frac{2}{2\theta - 1} \left( \frac{\theta}{e(2 - 2^\theta)} \right)^{\theta / (1-\theta)}\right) .$\\
Note that we use Lemma \ref{smale2} to estimate the coefficient to obtain the last inequality above.
\end{proof}
\begin{proposition}\label{prop2}
Let $(z_t)_{t\in\mathbb{N}}$ on $(Z, \mathcal{B}(Z))$ be a strictly stationary Markov chain which is $\phi$-mixing at least exponentially fast (i.e., $\phi_t\leq Dr^t$ for some constant $D>0$ and $0<r<1$). Let $\theta \in \left(\frac{1}{2},1\right)$. Then under Assumptions \ref{A1} and \ref{A2} we have that
\begin{align*}
\mathbb{E} \Bigg[ 
    &2 \sum_{1\leq i < j \leq t} 
    \frac{1}{i^{\theta}} \frac{1}{j^\theta} 
    \prod_{k=i+1}^t \left( 1 - \frac{\alpha}{k^\theta} \right) 
    \prod_{l=j+1}^t \left( 1 - \frac{\alpha}{l^\theta} \right) 
    \left\langle \nabla V_i(\omega^*), \nabla V_j(\omega^*) \right\rangle 
\Bigg] \notag \\
&\leq\; 
4 \sigma^2 \frac{Dr}{1-r} \, C_{\theta}
\left( \frac{1}{\alpha} \right)^{\frac{\theta}{1-\theta}} 
\left( \frac{1}{t+1} \right)^{\theta},
\end{align*}
where $ C_{\theta}= \left(8 + \frac{2}{2\theta - 1} \left( \frac{\theta}{e(2 - 2^\theta)} \right)^{\theta / (1-\theta)}\right).$
\end{proposition}
\begin{proof}
To simplify the notation, let us first denote \[I_{i+1}^t I_{j+1}^t=\prod_{k=i+1}^t \left( 1 - \frac{\alpha}{k^\theta} \right) 
\prod_{l=j+1}^t \left( 1 - \frac{\alpha}{l^\theta} \right).\] Hence,
\begin{align*}
\mathbb{E}& \Bigg[ 2 \sum_{1\leq i < j \leq t}  \frac{1}{i^{\theta}}\frac{1}{j^\theta} 
\prod_{k=i+1}^t \left( 1 - \frac{\alpha}{k^\theta} \right) 
\prod_{l=j+1}^t \left( 1 - \frac{\alpha}{l^\theta} \right) 
\left\langle \nabla V_i(\omega^*), \nabla V_j(\omega^*) \right\rangle \Bigg] \\
&= 2 \sum_{1\leq i < j \leq t} \frac{1}{i^{\theta}}\frac{1}{j^\theta} I_{i+1}^t I_{j+1}^t \mathbb{E}\left[\left\langle \nabla V_i(\omega^*), \nabla V_j(\omega^*) \right\rangle \right] \\
&= 2 \sum_{1\leq i < j \leq t}\frac{1}{i^{\theta}}\frac{1}{j^\theta} I_{i+1}^t I_{j+1}^t\int_Z \left( \int_Z \langle \nabla V_i(w^*), \nabla V_j(w^*) \rangle P^{j-i}(z_i, dz_j)\right) \,\rho(dz_i)  \\
&= 2 \sum_{1\leq i < j \leq t}\frac{1}{i^{\theta}}\frac{1}{j^\theta} I_{i+1}^t I_{j+1}^t  
\int_Z \left[ \int_Z \langle \nabla V_i(w^*), \nabla V_j(w^*) \rangle \left( P^{j-i}(z_i, dz_j) 
    -  \, d\rho(z_j)\right) \right] \,\rho(dz_i) \\
    &\leq 2 \sum_{1\leq i < j \leq t}\frac{1}{i^{\theta}}\frac{1}{j^\theta} I_{i+1}^t I_{j+1}^t  
\int_Z \left[\left | \int_Z \langle \nabla V_i(w^*), \nabla V_j(w^*) \rangle \left( P^{j-i}(z_i, dz_j) 
    -  \, d\rho(z_j)\right) \right|\right] \,\rho(dz_i) \\
&\leq 2 \sum_{1\leq i < j \leq t}\frac{1}{i^{\theta}}\frac{1}{j^\theta} I_{i+1}^t I_{j+1}^t 
\int_Z\left[ \int_Z\left|\langle \nabla V_i(w^*), \nabla V_j(w^*) \rangle \right|
    \,d \left|P^{j-i}_{z_i} - \rho\right|(z_j) \right] \,\rho(dz_i)  \\
    &\leq 2 \sigma^2 \sum_{1\leq i < j \leq t} \frac{1}{i^{\theta}}\frac{1}{j^\theta} I_{i+1}^t I_{j+1}^t  
 \int_Z  2\sup\{ P^{j-i}_{z_i}(A)-\rho(A):A\in \mathcal{B}(Z)\}\,\rho(dz_i)\\
&= 4 \sigma^2 \sum_{1\leq i < j \leq t} \frac{1}{i^{\theta}}\frac{1}{j^\theta} I_{i+1}^t I_{j+1}^t  
 \int_Z  d_{\text{TV}} \big( P^{j-i}_{z_i}, \rho \big)  \,\rho(dz_i)\\
&\leq 4 \sigma^2  \sum_{1\leq i < j \leq t}\frac{1}{i^{\theta}}\frac{1}{j^\theta} I_{i+1}^t I_{j+1}^t  \, \esssup_{z_i \in Z} \, 
    \Big[ d_{\text{TV}} \big( P^{j-i}_{z_i}, \rho \big) \Big] \int\rho(dz_i) \\
&= 4 \sigma^2\sum_{1\leq i < j \leq t}\frac{1}{i^{\theta}}\frac{1}{j^\theta} I_{i+1}^t I_{j+1}^t \, \esssup_{z_i \in Z} \, 
    \Big[ d_{\text{TV}} \big( P^{j-i}_{z_i}, \rho) \Big] \\
&= 4 \sigma^2\sum_{1\leq i < j \leq t}\frac{1}{i^{\theta}}\frac{1}{j^\theta} I_{i+1}^t I_{j+1}^t \phi_{j-i}.
\end{align*}
Note that in the sequence of inequalities presented above, we rely on the fact that \(\nabla V_z\) is Bochner integrable which allows us to interchange the inner product with the integral (see for eg., \cite{MR453964}), leading to the equality 
\[
\langle \nabla V_i(w^\star), \int_Z \nabla V_j(w^\star) \, d\rho(z_j) \rangle = \int_Z \langle \nabla V_i(w^\star), \nabla V_j(w^\star) \rangle \, d\rho(z_j),
\] where we are given that \(\int_Z \nabla V_j(w^\star)\,d\rho(z_j)=0\) i.e., \(\mathbb{E}[\nabla V(w^{\star})]=0\). Additionally, we also utilize the following property that for a finite signed measure \(\mu\) on a measurable space \((Z,\mathcal{B}(Z))\) and for any bounded measurable function \(f:Z\rightarrow \mathbb{R}\), \(\left|\int f\,d\mu\right|\leq \int |f|\,d|\mu|\), where \(|\mu|\) is precisely the total variation measure associated with the signed measure \(\mu\). For more details, see, e.g., Athreya and Lahiri \cite{MR2247694}.\\
Hence based on the above inequality we further establish that\\
\begin{align*}   
 \mathbb{E}& \Bigg[ 2 \sum_{1\leq i < j \leq t} \frac{1}{i^{\theta}}\frac{1}{j^\theta} 
\prod_{k=i+1}^t \left( 1 - \frac{\alpha}{k^\theta} \right) 
\prod_{l=j+1}^t \left( 1 - \frac{\alpha}{l^\theta} \right) 
\left\langle \nabla V_i(\omega^*), \nabla V_j(\omega^*) \right\rangle \Bigg]  \\
&\leq 4 \sigma^2 \sum_{1\leq i < j \leq t} \frac{1}{i^{\theta}}\frac{1}{j^\theta} 
\prod_{k=i+1}^t \left( 1 - \frac{\alpha}{k^\theta} \right) 
\prod_{l=j+1}^t \left( 1 - \frac{\alpha}{l^\theta} \right) \phi_{j-i} \\
&\leq 4 \sigma^2 \sum_{1\leq i < j \leq t} \frac{1}{i^{\theta}}\frac{1}{j^{\theta}}  e^{2 \alpha' ((i+1)^{1-\theta}-(t+1)^{1-\theta})}e^{2 \alpha' ((j+1)^{1-\theta}-(t+1)^{1-\theta})}\phi_{j-i}\\
&= 4 \sigma^2 e^{-4 \alpha' (t + 1)^{1-\theta}} 
\sum_{1\leq i < j \leq t} \frac{1}{i^{\theta}} \frac{1}{j^\theta}
e^{2 \alpha' (i+1)^{1-\theta}} e^{2 \alpha' (j+1)^{1-\theta}} \phi_{j-i},
\end{align*}
where \(\alpha'=\frac{\alpha}{1-\theta}\).
To derive the last inequality, we apply the coefficient estimates provided by Lemma \ref{smale2}. 
Now, we upper-bound the term \[\sum_{1\leq i < j \leq t} \frac{1}{i^{\theta}} \frac{1}{j^\theta}
e^{2 \alpha' (i+1)^{1-\theta}} e^{2 \alpha' (j+1)^{1-\theta}} \phi_{j-i},\]
which upon expanding, rearranging and finally using the fact that for $i < j$, $\frac{1}{j^\theta} < \frac{1}{i^\theta}$ gives
\begin{align*}
\sum_{1\leq i < j \leq t} \frac{1}{i^{\theta}} \frac{1}{j^\theta}
e^{2 \alpha' (i+1)^{1-\theta}} e^{2 \alpha' (j+1)^{1-\theta}} \phi_{j-i}
&\leq e^{-2 \alpha' (t+1)^{1-\theta}} \sum_{i=1}^t \frac{1}{i^{2\theta}} e^{2 \alpha' (i+1)^{1-\theta}}\sum_{i=1}^t\phi_i.
\end{align*}
Hence, 
\begin{align*}
\mathbb{E}& \Bigg[ 2 \sum_{1\leq i < j \leq t}  \frac{1}{i^{\theta}}\frac{1}{j^\theta} 
\prod_{k=i+1}^t \left( 1 - \frac{\alpha}{k^\theta} \right) 
\prod_{l=j+1}^t \left( 1 - \frac{\alpha}{l^\theta} \right) 
\left\langle \nabla V_i(\omega^*), \nabla V_j(\omega^*) \right\rangle \Bigg] \\
& \leq4 \sigma^2 e^{-4 \alpha' (t + 1)^{1-\theta}} 
\sum_{1\leq i < j \leq t} \frac{1}{i^{\theta}} \frac{1}{j^\theta}
e^{2 \alpha' (i+1)^{1-\theta}} e^{2 \alpha' (j+1)^{1-\theta}} \phi_{j-i}\\
&\leq 4 \sigma^2 e^{-2 \alpha' (t+1)^{-\theta}}\sum_{i=1}^t \frac{1}{i^{2\theta}} e^{2 \alpha' (i+1)^{1-\theta}}\sum_{i=1}^t\phi_i\\
&\leq 4 \sigma^2C_{\theta}\left( \frac{1}{\alpha} \right)^{\theta / (1-\theta)} \left( \frac{1}{t+1} \right)^{\theta}\frac{Dr}{1-r},
\end{align*}
where $ C_{\theta}= \left(8 + \frac{2}{2\theta - 1} \left( \frac{\theta}{e(2 - 2^\theta)} \right)^{\theta / (1-\theta)}\right).$\\
Note that we use Lemma \ref{smale2} (for more details also see Smale and Yao \cite{MR2228737}) to estimate the coefficient and the fact that \((z_t)_{t\in \mathbb{N}}\) is exponentially \(\phi\)-mixing to finally obtain the last inequality above.
\end{proof}
\begin{remark}
    Note that for the case \(\theta=1\), the proof follows similar arguments to that of Propositions \ref{prop1} and \ref{prop2} but by applying Lemma \ref{smale2} for the case \(\theta=1\) and taking into consideration the different values of \(\alpha\in (0,1]\).
\end{remark}
Building on the ideas of Samle and Yao~\cite{MR2228737}, we now present the following result.
\begin{theorem} \label{thm:1}
Let us consider a strictly stationary Markov chain \((z_t)_{t\in\mathbb{N}}\) such that it is \(\phi\)-mixing at least exponentially fast i.e., there exist constants \( D > 0 \) and \( 0 < r < 1 \) such that \( \beta_t \leq \phi_t \leq D r^t \). Furthermore,
let $\theta\in\left(\tfrac12,1\right)$ and consider $\gamma_t=\frac{1}{\eta t^\theta}$. Then under Assumptions \ref{A1} and \ref{A2}, for each $t\in \mathbb{N}$ and \(w_t\) obtained by Eq.~\eqref{eq:stograd}, we have
$$\|w_{t}-w^{\star}\| \leq \mathcal{E}_\text{init}(t)+\mathcal{E}_\text{samp}(t)$$
where,
$$\mathcal{E}_\text{init}(t)\leq e^{\frac{2\alpha}{1-\theta}(1-t^{1-\theta})}\|w_{1}-w^{\star}\|;$$
and with probability at least $1-\delta$, with $\delta \in(0,1)$ in the space $Z^{t-1}$,
$$\mathcal{E}^2_\text{samp}(t) \leq  \frac{\sigma^2C_{\theta}}{\delta \eta^2}\left(\frac{1}{\alpha} \right)^{\theta / (1-\theta)}\left( \frac{1}{t} \right)^{\theta}\left(1+\frac{4Dr}{1-r}\right),$$

with $ C_{\theta}= \left(8 + \frac{2}{2\theta - 1} \left( \frac{\theta}{e(2 - 2^\theta)} \right)^{\theta / (1-\theta)}\right).$
    \end{theorem}
    \begin{proof}
We define the residual at time \( t \) for \(t\geq 2\) as
\[
r_t = w_t - w^\star.
\]
Given the update rule
\[
w_{t+1} = w_t - \gamma_t (A_t w_t + B_t),
\]
we obtain
\[
r_{t+1} = (I - \gamma_t A_t) r_t - \gamma_t (A_t w^\star + B_t).
\]
Hence by induction,
\[
r_{t+1} = X_1^t r_1 - \sum_{i=1}^t \gamma_i X_{i+1}^t Y_i,
\]
where
\[
X_{i}^t = \prod_{k=i}^t (I - \gamma_k A_k)~\text{and}~Y_i = A_i w^\star + B_i.
\]
Note that \(X_{i+1}^t\) is a symmetric linear operator from \(W\) to \(W\) and depends on \( z_{i+1}, \ldots, z_t \). The term \( X_1^t w_1 \) represents the accumulated initial error, while the term \( \sum_{i=1}^t \gamma_i X_{i+1}^t Y_i \) represents the random fluctuations caused by sampling and has zero mean hence we define 
\[
\mathcal{E}_{\text{init}}(t+1) = \|X_1^t r_1\|,
\]
\[
\mathcal{E}_{\text{samp}}(t+1) = \left\| \sum_{i=1}^t \gamma_i X_{i+1}^t Y_i \right\|.
\]
Now with \( \kappa I \leq A_t \leq \eta I \) and \( \gamma_t = \frac{1}{\eta t^\theta} \), we get
\[
\|X_{1}^t r_1\| \leq \prod_{k=1}^t \left(1 - \frac{\alpha}{k^\theta}\right) \|r_1\|, \quad \alpha = \frac{\kappa}{\eta}\in (0,1].
\]
By applying Lemma \ref{smale1} in the above inequality, we obtain the result on the initial error.\\
As for the sampling error we obtain that
\begin{align*}
\left\| \sum_{i=1}^t \gamma_i X_{i+1}^t Y_i \right\|^2 
&= \sum_{i=1}^t \left\| \gamma_i X_{i+1}^t Y_i \right\|^2 
+ 2 \sum_{1\leq i < j \leq t} \langle \gamma_i X_{i+1}^t Y_i, \gamma_j X_{j+1}^t Y_j \rangle \\
&\leq \frac{1}{\eta^2}\sum_{i=1}^t \frac{1}{i^{2\theta}} \prod_{k=i+1}^t \left(1 - \frac{\alpha}{k^\theta} \right)^2 \|\nabla V_i(w^\star)\|^2 \\
&\quad +  \frac{2}{\eta^2}\sum_{1\leq i < j \leq t} \frac{1}{i^{\theta}} \frac{1}{j^\theta} 
\prod_{k=i+1}^t \left(1 - \frac{\alpha}{k^\theta} \right) 
\prod_{l=j+1}^t \left(1 - \frac{\alpha}{l^\theta} \right) 
\langle \nabla V_i(w^\star), \nabla V_j(w^\star) \rangle.
\end{align*}
Finally, applying Propositions \ref{prop1} and \ref{prop2} to the above equality, we may conclude
\[
\mathbb{E}\left[\left\| \sum_{i=1}^t \gamma_i X_{i+1}^t Y_i \right\|^2 \right] 
\leq \frac{\sigma^2 C_{\theta}}{\eta^2} \left( \frac{1}{t+1} \right)^\theta +\frac{4 \sigma^2C_{\theta}}{\eta^2} \left( \frac{1}{t+1} \right)^{\theta}\frac{Dr}{1-r}.
\]
Using Markov's inequality, which states that for a non-negative random variable \( X \) and any \( \epsilon > 0 \)
\[
\mathbb{P}(X \geq \epsilon) \leq \frac{\mathbb{E}[X]}{\epsilon},
\]
we obtain for \( X = \mathcal{E}^2_{\text{samp}}(t) \) and \( t \geq 2 \) that
\begin{align*}
    \mathbb{P}(\mathcal{E}^2_{\text{samp}}(t)>\epsilon^2)&\leq \frac{\mathbb{E}[\mathcal{E}^2_{\text{samp}}(t)]}{\epsilon^2}\\
\text{which further implies that}~\mathbb{P}(\mathcal{E}^2_{\text{samp}}(t)\leq \epsilon^2)&\geq 1-\frac{\mathbb{E}[\mathcal{E}^2_{\text{samp}}(t)]}{\epsilon^2}.
\end{align*}
Taking $\delta=\frac{\mathbb{E}[\mathcal{E}^2_{\text{samp}}(t)]}{\epsilon^2}$, we get a probabilistic upper bound on the sample error i.e., with probability at least $1-\delta$, with $\delta \in(0,1)$ in the space $Z^{t-1}$,
$$\mathcal{E}^2_\text{samp}(t) \leq  \frac{\sigma^2C_{\theta}}{\delta\eta^2}\left( \frac{1}{\alpha} \right)^{\theta / (1-\theta)}\left( \frac{1}{t} \right)^{\theta}\left(1+\frac{4Dr}{1-r}\right),$$
with $ C_{\theta}= \left(8 + \frac{2}{2\theta - 1} \left( \frac{\theta}{e(2 - 2^\theta)} \right)^{\theta / (1-\theta)}\right).$
\end{proof}
\begin{remark}
    Note that for $t=1$, $\mathcal{E}_\text{init}(t)=\|w_{1}-w^{\star}\|$ and $\mathcal{E}^2_\text{samp}(t)=0.$
\end{remark}
Using arguments similar to those presented in the proof of Theorem~\ref{thm:1}, we provide a concise proof of Proposition~\ref{rm:1} below.
\begin{proposition} \label{rm:1}
    With all the assumptions of Theorem \ref{thm:1}, but with $\theta=1$ and $\alpha\in \left(0,\frac{1}{2}\right)$, we obtain that
    $$\|w_{t}-w^{\star}\|\leq  \mathcal{E}_\text{init}(t)+\mathcal{E}_\text{samp}(t)$$
    where,
$$\mathcal{E}_\text{init}(t)\leq \left(\frac{1}{t}\right)^\alpha\|w_{1}-w^{\star}\|;$$
and with probability at least $1-\delta,~\text{with}~\delta \in(0,1)$ in the space $Z^{t-1}$,
$$\mathcal{E}^2_\text{samp}(t)\leq  \frac{4\sigma^2}{\delta\eta^2}\left(\frac{1}{1-2\alpha}\right)\left(\frac{1}{t}\right)^\alpha\left(1+\frac{6Dr}{1-r}\right).$$
\end{proposition}
\begin{proof}
Following the proof similar to that of Theorem \ref{thm:1}, we have observed that
\begin{align*}
\left\| \sum_{i=1}^t \gamma_i X_{i+1}^t Y_i \right\|^2 
&= \sum_{i=1}^t \left\| \gamma_i X_{i+1}^t Y_i \right\|^2 
+ 2 \sum_{1\leq i < j \leq t} \langle \gamma_i X_{i+1}^t Y_i, \gamma_j X_{j+1}^t Y_j \rangle \\
&\leq \frac{1}{\eta^2}\sum_{i=1}^t \frac{1}{i^{2\theta}} \prod_{k=i+1}^t \left(1 - \frac{\alpha}{k^\theta} \right)^2 \|\nabla V_i(w^\star)\|^2 \\
&\quad +  \frac{2}{\eta^2}\sum_{1\leq i < j \leq t} \frac{1}{i^{\theta}} \frac{1}{j^\theta} 
\prod_{k=i+1}^t \left(1 - \frac{\alpha}{k^\theta} \right) 
\prod_{l=j+1}^t \left(1 - \frac{\alpha}{l^\theta} \right) 
\langle \nabla V_i(w^\star), \nabla V_j(w^\star) \rangle.
\end{align*}
Hence we apply Lemma \ref{smale2} for the case \(\theta =1\) and \(\alpha\in\left(0,\frac{1}{2}\right)\) in Propositions \ref{prop1} and \ref{prop2} to observe that \(\mathbb{E}\left[\left\| \sum_{i=1}^t \gamma_i X_{i+1}^t Y_i \right\|^2\right]\) is bounded above by 
\begin{align*}  \frac{4\sigma^2}{\eta^2(1 - 2\alpha)} \left(\frac{1}{t+1}\right)^{2\alpha}+\frac{4 \sigma^2}{\eta^2}\sum_{1\leq i < j \leq t} \frac{1}{i}\frac{1}{j} \prod_{k=i+1}^t \left( 1 - \frac{\alpha}{k}\right)  
\prod_{l=j+1}^t \left( 1 - \frac{\alpha}l \right)\phi_{j-i}.
\end{align*}
For the case \(\theta=1\), we use Lemma \ref{smale1} to observe that \[\sum_{1\leq i < j \leq t} \frac{1}{i}\frac{1}{j} \prod_{k=i+1}^t \left( 1 - \frac{\alpha}{k}\right)  
\prod_{l=j+1}^t \left( 1 - \frac{\alpha}{l} \right)\phi_{j-i}\] is bounded above by 
\begin{align*}
\left(\frac{1}{t+1}\right)^{2\alpha}\sum_{1\leq i < j \leq t} \frac{1}{i}\frac{1}{j}(i+1)^\alpha(j+1)^\alpha \phi_{j-i}.
\end{align*}
 Expanding and rearranging the above inequality along with later using the fact that for $i < j$, $\frac{1}{j} < \frac{1}{i}$ we further obtain that 
 \[\left(\frac{1}{t+1}\right)^{2\alpha}\sum_{1\leq i < j \leq t} \frac{1}{i}\frac{1}{j}(i+1)^\alpha(j+1)^\alpha \phi_{j-i}\leq \sum_{i=1}^t\frac{1}{i^2}\left(\frac{i+1}{t+1}\right)^\alpha\sum_{i=1}^t\phi_i.\]
Applying Lemma \ref{myineq} and and the fact that \((z_t)_{t\in\mathbb{N}}\) is \(\phi\)-mixing at least exponentially fast, 
\begin{align*}
    \sum_{i=1}^t\frac{1}{i^2}\left(\frac{i+1}{t+1}\right)^\alpha\sum_{i=1}^t\phi_i\leq \frac{6}{1-\alpha}\left(\frac{1}{t}\right)^\alpha \left(\frac{Dr}{1-r}\right).
\end{align*}
Hence, for \(\theta=1\) and \(\alpha\in\left(0,\frac{1}{2}\right)\) along with applying Markov inequality for  \( X = \mathcal{E}^2_{\text{samp}}(t) \) we obtain that with probability at least $1-\delta$, with $\delta \in(0,1)$ in the space $Z^{t-1}$,
\begin{align*}
    \mathcal{E}^2_\text{samp}(t) &\leq  \frac{4\sigma^2}{\delta\eta^2}\left[\frac{1}{(1-2\alpha)}\left(\frac{1}{t}\right)^{2\alpha} +\frac{6Dr}{(1-\alpha)(1-r)}\left(\frac{1}{t}\right)^{\alpha}\right]\\
    &\leq \frac{4\sigma^2}{\delta\eta^2}\left(\frac{1}{1-2\alpha}\right)\left(\frac{1}{t}\right)^\alpha\left(1+\frac{6Dr}{1-r}\right). 
\end{align*}
\end{proof}
\begin{remark}\label{remphimix}
Based on Theorem \ref{thm:1} and Proposition \ref{rm:1}, in the finite sample setting, we observe that the exponential decay rate determined by the mixing coefficient influences the error bounds we obtain, in contrast to the i.i.d.\ case treated by Smale and Yao \cite{MR2228737}. When \( r \to 0 \), the process mixes more rapidly and the factor \( \tfrac{r}{1 - r} \) tends to zero, so our error bounds essentially coincide with those for independent samples, and hence our bounds are sharper in this sense. Stronger dependence reduces the effective information and leads to worse error bounds compared to i.i.d.
Turning to the convergence rates, when the parameter \( \theta \in \left( \tfrac{1}{2}, 1 \right) \), we still achieve the rate \( \mathcal{O}\!\left( t^{-\theta/2} \right) \), matching the i.i.d.\ rate of \cite[Theorem B]{MR2228737}. However, at the boundary value \( \theta = 1 \), the rate slows to \( \mathcal{O}\!\left( t^{-\alpha/2} \right) \), whereas the i.i.d.\ scenario attains the faster rate, \( \mathcal{O}\!\left( t^{-\alpha} \right) \) as shown in \cite[Remark 3.7]{MR2228737} where, \( \alpha  \in (0,1] \).
\end{remark}
Following the previous analysis, we now extend our approach to examine a \(\beta\)-mixing sequence.
\section{\(\beta\)-mixing sequence}\label{betamix}
\begin{proposition}\label{beta1}
Let $(z_t)_{t\in\mathbb{N}}$ on $(Z, \mathcal{B}(Z))$ be a strictly stationary Markov chain such that the chain is $\beta$-mixing at least exponentially fast (i.e., $\beta_t\leq D_1r_1^t$ for some constant $D_1>0$ and $0<r_1<1$). Let $\theta \in \left(\frac{1}{2},1\right)$. Then under Assumptions \ref{A1} and \ref{A2} we have that
\begin{align*}
\mathbb{E} \Bigg[ 
    &2 \sum_{1\leq i < j \leq t} 
    \frac{1}{i^{\theta}} \frac{1}{j^\theta} 
    \prod_{k=i+1}^t \left( 1 - \frac{\alpha}{k^\theta} \right) 
    \prod_{l=j+1}^t \left( 1 - \frac{\alpha}{l^\theta} \right) 
    \left\langle \nabla V_i(\omega^*), \nabla V_j(\omega^*) \right\rangle 
\Bigg] \notag \\
&\leq\; 
4 \sigma^2 \frac{D_1 r_1}{1 - r_1} \, C_{\theta}
\left( \frac{1}{\alpha} \right)^{\frac{\theta}{1 - \theta}} 
\left( \frac{1}{t+1} \right)^{\theta},
\end{align*}
where $ C_{\theta}= \left(8 + \frac{2}{2\theta - 1} \left( \frac{\theta}{e(2 - 2^\theta)} \right)^{\theta / (1-\theta)}\right).$
\end{proposition}
\begin{proof}
Applying Proposition \ref{prop2}, we observe that
\begin{align*}
\mathbb{E}& \Bigg[ 2 \sum_{1\leq i < j \leq t}  \frac{1}{i^{\theta}}\frac{1}{j^\theta} 
\prod_{k=i+1}^t \left( 1 - \frac{\alpha}{k^\theta} \right) 
\prod_{l=j+1}^t \left( 1 - \frac{\alpha}{l^\theta} \right) 
\left\langle \nabla V_i(\omega^*), \nabla V_j(\omega^*) \right\rangle \Bigg] \\
&\leq 4 \sigma^2 \sum_{1\leq i < j \leq t} \frac{1}{i^{\theta}}\frac{1}{j^\theta} I_{i+1}^t I_{j+1}^t  
 \int_Z  d_{\text{TV}} \big( P^{j-i}_{z_i}, \rho \big)  \,\rho(dz_i)\\
&= 4 \sigma^2  \sum_{1\leq i < j \leq t}\frac{1}{i^{\theta}}\frac{1}{j^\theta} I_{i+1}^t I_{j+1}^t \beta_{j-i}.\\
\end{align*}
Now following the rest of the arguments similar to that of Proposition \ref{prop2}, we obtain the desired result.
\end{proof}
We now apply Proposition \ref{prop1} and Proposition \ref{beta1} presented above to the sampling error
\begin{align*}
\left\| \sum_{i=1}^t \gamma_i X_{i+1}^t Y_i \right\|^2 
&\leq \frac{1}{\eta^2}\sum_{i=1}^t \frac{1}{i^{2\theta}} \prod_{k=i+1}^t \left(1 - \frac{\alpha}{k^\theta} \right)^2 \|\nabla V_i(w^\star)\|^2 \\
&\quad +  \frac{2}{\eta^2}\sum_{1\leq i < j \leq t} \frac{1}{i^{\theta}} \frac{1}{j^\theta} 
\prod_{k=i+1}^t \left(1 - \frac{\alpha}{k^\theta} \right) 
\prod_{l=j+1}^t \left(1 - \frac{\alpha}{l^\theta} \right) 
\langle \nabla V_i(w^\star), \nabla V_j(w^\star) \rangle,
\end{align*} 
where the rest of the arguments follow similarly to the way presented in the proof of Theorem \ref{thm:1} to finally obtain the following result
\begin{theorem} \label{beta}
Let us consider a strictly stationary Markov chain \((z_t)_{t\in\mathbb{N}}\) such that it is exponentially \(\beta\)-mixing i.e., there exist constants \( D_1 > 0 \) and \( 0 < r_1 < 1 \) such that \( \beta_t \leq D_1 r_1^t \). Furthermore,
let $\theta\in\left(\tfrac12,1\right)$ and consider $\gamma_t=\frac{1}{\eta t^\theta}$. Then under Assumptions \ref{A1} and \ref{A2}, for each $t\in \mathbb{N}$ and \(w_t\) obtained by Eq.~\eqref{eq:stograd}, we have
$$\|w_{t}-w^{\star}\| \leq \mathcal{E}_\text{init}(t)+\mathcal{E}_\text{samp}(t)$$
where,
$$\mathcal{E}_\text{init}(t)\leq e^{\frac{2\alpha}{1-\theta}(1-t^{1-\theta})}\|w_{1}-w^{\star}\|;$$
and with probability at least $1-\delta$, with $\delta \in(0,1)$ in the space $Z^{t-1}$,
$$\mathcal{E}^2_\text{samp}(t) \leq  \frac{\sigma^2C_{\theta}}{\delta' \eta^2}\left( \frac{1}{\alpha} \right)^{\theta / (1-\theta)}\left( \frac{1}{t} \right)^{\theta}\left(1+\frac{4D_1r_1}{1-r_1}\right),$$
with $ C_{\theta}= \left(8 + \frac{2}{2\theta - 1} \left( \frac{\theta}{e(2 - 2^\theta)} \right)^{\theta / (1-\theta)}\right).$
    \end{theorem}
    \begin{remark}
        The same learning rates and finite sample error bounds discussed in Remark \ref{remphimix} continue to hold here; the only variation is that the dependence must now decay at a rate quantified by the \(\beta\)‐mixing coefficients.
    \end{remark}
   \section{Learning rates based on polynomial decay of mixing coefficients}\label{poly}
   We now infer some results on the learning rates of SS-MGD based on the polynomial decay of $\phi$-mixing sequences.
   \begin{theorem}
       Let us consider a strictly stationary Markov chain \((z_t)_{t\in\mathbb{N}}\) such that it is \(\phi\)-mixing satisfying a polynomial decay, i.e., \(\phi_t\leq bt^{-k}\) for \(b>0\) and \(k>0\). Furthermore,
let $\theta\in\left(\tfrac12,1\right)$ and consider $\gamma_t=\frac{1}{\eta t^\theta}$. Then under Assumptions \ref{A1} and \ref{A2}, for each $t\in \mathbb{N}$ and \(w_t\) obtained by Eq.~\eqref{eq:stograd}, we have
$$\|w_{t}-w^{\star}\|
 =
\begin{cases}
O\left(t^{\frac{1-k-\theta}{2}}\right), & 0<k<1,\\[8pt]
O\left(t^{-\theta/2} (\log t)^{1/2} \right), & k=1,\\[8pt]
O\left(t^{-\theta/2}\right), & k>1.
\end{cases}$$
    \end{theorem}
  \begin{proof}
   We proceed as in the proof of Theorem~\ref{thm:1}, applying Propositions~\ref{prop1} and~\ref{prop2} to obtain
\[
\|w_t - w^\star\| \leq \mathcal{E}_{\mathrm{init}}(t) + \mathcal{E}_{\mathrm{samp}}(t),
\]
where
\[
\mathcal{E}_{\mathrm{init}}(t) \leq 
\exp\left(\frac{2\alpha}{1-\theta} (1 - t^{1-\theta})\right) \, \|w_1 - w^\star\|;
\]
and, with probability at least \(1 - \delta\), where \(\delta \in (0,1)\), over the sample space \(Z^{t-1}\), we have
\[
\mathcal{E}_{\mathrm{samp}}^2(t) \leq 
\frac{\sigma^2 C_\theta}{\delta \eta^2}
\left(\frac{1}{\alpha}\right)^{\frac{\theta}{1-\theta}}
\left(\frac{1}{t}\right)^{\theta}
\left(1 + \sum_{i=1}^{t} \phi_i \right),
\]
where
\[
C_\theta = 8 + \frac{2}{2\theta - 1}
\left( \frac{\theta}{e(2 - 2^\theta)} \right)^{\frac{\theta}{1-\theta}}.
\]

Since the \(\phi\)-mixing coefficients decay polynomially, there exists a constant \(b > 0\) such that
\[
\sum_{i=1}^{t} \phi_i \leq b \sum_{i=1}^{t} i^{-k}.
\]
Recall that
\[
\sum_{i=1}^{t} i^{-k} =
\begin{cases}
O(t^{1-k}), & 0 < k < 1, \\[6pt]
O(\log t), & k = 1, \\[6pt]
O(1), & k > 1.
\end{cases}
\]

Combining these bounds yields
\[
\|w_t - w^\star\| =
\begin{cases}
O\left(t^{\frac{1-k-\theta}{2}}\right), & 0 < k < 1, \\[8pt]
O\left(t^{-\theta/2} (\log t)^{1/2} \right), & k = 1, \\[8pt]
O\left(t^{-\theta/2} \right), & k > 1.
\end{cases}
\]
  \end{proof}
Similar conclusions follow for \(\beta\)-mixing sequences with polynomial decay.
\section{Conclusion}
In this paper, we consider a stationary Markov chain exhibiting exponential $\phi$- and $\beta$-mixing properties, characterized by the bounds $\phi_t \leq Dr^t$ and $\beta_t \leq D_1r_1^t$, where $D, D_1 > 0$ and $0 < r, r_1 < 1$. For such processes, we establish upper bounds based on the strictly stationary Markov chain stochastic gradient algorithm.  The exponential decay of the mixing coefficients, reflected in the parameters $r$ and $r_1$, plays a crucial role in determining the rate of decay of dependence. Specifically, for the case of finite samples, smaller values of $r$ and $r_1$ result in upper bounds for the sampling error that closely approach those of an i.i.d.~sample.
\par
Moreover, in all derived results, the upper bounds are expressed as the sum of the i.i.d.\ sample bound and an additional term that captures the effect of the dependence structure of the process. Consequently, the upper bound for an i.i.d.\ sample arises as a special case of our general result, corresponding to the case where $D=0$ respectively $D_1=0$ .
\par 
Building on our current work, there are several open problems that need further study. For instance, to extend our analysis to find upper bounds on the learning performance of an online regularized learning algorithm based on mixing sequences. Another important direction is to explore how these results can be extended to handle non-stationary mixing sequences.
\section*{Acknowledgements}
This research was carried out under the Austrian COMET program (project S3AI with FFG no. 872172, \url{www.S3AI.at}, at SCCH, \url{www.scch.at}),  which is funded by the Austrian ministries BMK, BMDW, and the province of Upper Austria and the Federal Ministry for Climate Action, Environment, Energy, Mobility, Innovation and Technology (BMK),the Federal Ministry for Labour and Economy (BMAW), and the State of Upper Austria in the frame of the SCCH competence center INTEGRATE [(FFG grant no. 892418)] in the COMET - Competence Centers for Excellent Technologies Programme managed by Austrian ResearchPromotion Agency FFG.

\appendix
We now recall some well established results on the estimates of the coefficients relevant to our proofs. For the sake of completeness, we briefly summarize them here
\section{Some established inequalities}\label{appen}
\begin{lemma}[\cite{MR2228737}]\label{smale1}For $\alpha\in(0,1]$ and $\theta \in(0,1)$, we observe that \[\prod_{k=i+1}^{t}\biggl(1-\frac{\alpha}{ k^\theta}\biggl)\leq e^{\biggl(\frac{2\alpha}{1-\theta}((i+1)^{1-\theta}-(t+1)^{1-\theta})\biggl)},\] and for $\theta=1$,
 $$\prod_{k=i+1}^{t}\biggl(1-\frac{\alpha}{ k}\biggl)\leq \biggl(\frac{i+1}{t+1}\biggl)^{\alpha}.$$
\end{lemma}
For specific choices of $\theta$ and $\alpha$, the function \[\psi_\theta(t, \alpha)=\sum_{i=1}^{t-1} \frac{1}{i^{2\theta}} \prod_{k=i+1}^{t-1} \left( 1 - \frac{\alpha}{k^\theta} \right)^2\] admits the following upper bounds
\begin{lemma}[\cite{MR2228737}]\label{smale2}Let \( \alpha \in (0, 1] \) and \( \theta \in \left( \frac{1}{2}, 1 \right) \). Then for \( t \in \mathbb{N} \),
\[
\psi_\theta(t+1, \alpha) \leq C_\theta \left( \frac{1}{\alpha} \right)^{\theta / (1-\theta)} \left( \frac{1}{t+1} \right)^\theta,
\]
where
\[
C_\theta = 8 + \frac{2}{2\theta - 1} \left( \frac{\theta}{e(2 - 2^\theta)} \right)^{\theta / (1-\theta)}.
\]
If \( \theta = 1 \), and for \( \alpha \in (0, 1] \),
\[
\psi_1(t+1, \alpha) = \sum_{i=1}^t \frac{1}{i^2} \prod_{k=i+1}^t \left( 1 - \frac{\alpha}{k} \right)^2
\]
\[
\leq
\begin{cases} 
\frac{4}{1 - 2\alpha} (t+1)^{-2\alpha}, & \alpha \in (0, \frac{1}{2}), \\[8pt]
4(t+1)^{-1} \ln(t+1), & \alpha = \frac{1}{2}, \\[8pt]
\frac{6}{2\alpha - 1}(t+1)^{-1}, & \alpha \in (\frac{1}{2}, 1), \\[8pt]
6(t+1)^{-1}, & \alpha = 1.
\end{cases}
\]
\end{lemma}
\begin{lemma}\label{myineq}
    Let \(t \in \mathbb{N}\) and \(\alpha \in (0, 1]\). Then, the following bounds hold 
\[
\sum_{i=1}^t \frac{1}{i^2} \left(\frac{i+1}{t+1}\right)^{\alpha} \leq
\begin{cases}
\frac{6}{1-\alpha}(t+1)^{-\alpha}, & \text{if } \alpha \in (0,1), \\[8pt]
\frac{6 \ln (t+1)}{t+1}, & \text{if } \alpha = 1.
\end{cases}
\]
\end{lemma}
\begin{proof}
    We start with the inequality
\begin{align}\label{ineq}
   \sum_{i=1}^t \frac{1}{i^2} \left(\frac{i+1}{t+1}\right)^{\alpha} \leq t^{-2} + \sum_{i=1}^{t-1} \frac{1}{i^2} \left(\frac{i+1}{t+1}\right)^{\alpha}. 
\end{align}
To simplify the summation, note that for \(i\in \mathbb{N}\), we have \(i+1 \leq 2k\), so \(\frac{i+1}{t+1} \leq \frac{2i}{t+1}\). Applying this bound, we write
\[
\sum_{i=1}^{t-1} \frac{1}{i^2} \left(\frac{i+1}{t+1}\right)^{\alpha} \leq \frac{2^{\alpha}}{(t+1)^{\alpha}} \sum_{i=1}^{t-1} i^{\alpha - 2}.
\]
Since \(i^{\alpha - 2}\) is a decreasing function for \(\alpha \in (0,1]\), we have
\[
\sum_{i=1}^{t-1} i^{\alpha - 2} \leq \int_{1/2}^{t - 1/2} x^{\alpha - 2} \, dx.
\]
If \(\alpha \in (0,1)\), then
\[
\int_{1/2}^{t - 1/2} x^{\alpha - 2} \, dx = \frac{(t - 1/2)^{\alpha - 1} - (1/2)^{\alpha - 1}}{\alpha - 1}.
\]
Thus the right hand side of Eq.~\ref{ineq} is bounded above by
\begin{align*}
    t^{-2} + \frac{2^\alpha}{(t+1)^\alpha} \cdot \frac{(1/2)^{\alpha-1}}{1 - \alpha}&\leq \frac{4}{(t+1)^2}+\frac{2}{(1-\alpha)(t+1)^\alpha}\\
&\leq\left(4+\frac{2}{1-\alpha}\right)\frac{1}{(t+1)^\alpha}\\
&<\frac{6}{1-\alpha}\frac{1}{(t+1)^\alpha}.
\end{align*}
When \(\alpha = 1\) we obtain that
\[
\int_{1/2}^{t - 1/2} x^{-1} \, dx = \ln(t - 1/2) - \ln(1/2)< \ln(t+1) + \ln 2.
\]
Thus the right hand side of Eq.~\ref{ineq} is bounded above by

\begin{align*}
    t^{-2} + \frac{2}{(t+1)} \cdot \left(\ln(t+1) + \ln 2\right)&\leq \frac{4}{(t+1)^2}+\frac{2 \ln2}{(t+1)}\ln(t+1)\\
&\leq\frac{6}{(t+1)^2}\ln(t+1)+\frac{2 \ln2}{(t+1)}\ln(t+1)\\
&=\left(\frac{6}{t+1}+2\ln 2\right)\frac{\ln(t+1)}{t+1}\\
&<\frac{6\ln(t+1)}{t+1}.
\end{align*}
\end{proof}
\end{document}